\documentclass[preprint,review,10pt]{elsarticle}
\usepackage[utf8]{inputenc}
\usepackage{algpseudocode}
\usepackage{algorithm}
\usepackage{amsmath}
\usepackage{amssymb,textcomp,textcomp,amsthm}
\usepackage{graphicx}
\usepackage{geometry}
\usepackage{dsfont}
\usepackage[]{hyperref}
\usepackage{float}
\usepackage[lofdepth,lotdepth]{subfig}
\usepackage{tikz}
\usepackage{soul}
\usepackage[modulo]{lineno}

\usepackage{mathtools}

\DeclarePairedDelimiter\floor{\lfloor}{\rfloor}

\newcommand{\nop}[1]{}

\newtheorem{theorem}{Theorem}

\newtheorem{proposition}[theorem]{Proposition}

\newtheorem{definition}[theorem]{Definition}

\usetikzlibrary{arrows}

\tikzset{
  treenode/.style = {align=center, inner sep=0pt, text centered,
    font=\sffamily},
  arn_n/.style = {treenode, circle, white, font=\sffamily\bfseries, draw=black,
    fill=black, text width=1.5em},
  arn_r/.style = {treenode, rectangle, red, draw=blue, 
    minimum width=3.5em, minimum height=3.5em},
  arn_x/.style = {treenode, circle,red, draw=blue,
    minimum width=3.5em, minimum height=3.5em}
}

\begin{document}

\begin{frontmatter}

\title{Column generation based math-heuristic for classification trees}


\author{Murat Firat\tnoteref{label1}}
\tnotetext[label1]{{Corresponding author: {\tt m.firat@tue.nl}}, Eindhoven University of Technology, The Netherlands.}
\author{Guillaume Crognier\tnoteref{label2}}
\author{Adriana F. Gabor\tnoteref{label3}}
\author{C.A.J. Hurkens\tnoteref{label4}}
\author{Yingqian Zhang\tnoteref{label5}}
\tnotetext[label2]{\indent {{\tt guillaume.crognier@polytechnique.edu}}. École Polytechnique Paris, France.} 
\tnotetext[label3]{\indent {{\tt adriana.gabor@ku.ac.ae}}. Khalifa University, United Arab Emirates.}
\tnotetext[label4]{\indent {{\tt c.a.j.hurkens@tue.nl}}. Eindhoven University of Technology, The Netherlands.}
\tnotetext[label5]{\indent {{\tt yqzhang@tue.nl}}. Eindhoven University of Technology, The Netherlands.}

\begin{abstract}
 This paper explores the use of  Column Generation (CG)  techniques in constructing univariate binary decision trees for classification tasks. We propose a novel Integer Linear Programming (ILP) formulation, based on root-to-leaf paths in decision trees.  The model is solved via a Column Generation based heuristic. To speed up the heuristic, we use a restricted instance data by considering a subset of decision splits, sampled from the solutions of  the well-known CART algorithm. Extensive numerical experiments show  that our approach is competitive with  the state-of-the-art ILP-based algorithms. In particular, the proposed approach  is capable of handling big data sets with tens of thousands of data rows.
 Moreover, for large data sets, it finds solutions competitive to 
 CART. 
 
\end{abstract}


\end{frontmatter}

\section{Introduction}\label{introduction}

In classification problems, the goal is to decide the class membership of a set of  observations, by using available  information on features and class membership of a training data set. Decision trees are one of the most popular models for solving this problem,  due to their effectiveness  and high interpretability. They have been applied in a wide range of  applications, from transport planning  \cite{vanriessen2016}, internet advertisements \cite{kim2001}, to  healthcare interventions \cite{linden2018}.  In this work, we focus on constructing univariate binary decision trees of prespecified depth.

In a univariate binary decision tree, each internal node contains a test regarding the value of one single feature of the data set, while the leaves contain the target classes. The problem of constructing (learning) a classification  tree (CTCP), is the problem of finding a set of tests (decision checks), such that  the assignment of target classes to rows satisfies a certain criterion. A commonly encountered objective  is accuracy,  measured as the share of the number of correct predictions in a training set.

As the problem of learning optimal decision trees is an NP-complete problem \cite{HyaRi76}, heuristics such as CART \cite{Breiman84} and ID3 \cite{Quinlan86} are widely used. These greedy algorithms build a tree recursively, starting from a single node. At each internal node, the (locally) best decision split  is chosen  by solving an optimization problem  on a subset of the training data. This process is repeated at  children nodes till some stopping criterion is satisfied.  Although greedy algorithms are computationally efficient, they do not guarantee finding an  optimal tree. In recent years, constructing decision trees by using mathematical programming techniques, especially Integer Optimization, became a hot topic among researchers (see \citet{Menick16}, \citet{VerZhaYe17}, \citet{BertDunn17}, \citet{VerZha17}, and \citet{Dash18}). 

In this paper, our contribution is threefold. First, we propose a novel ILP formulation for constructing univariate classification trees and solve it through a Column Generation based heuristic. Secondly, we show that by using only a subset  of the feature checks (decision splits), solutions of good quality can be obtained within short computation time. Thirdly, we provide ILP solutions to problems involving large data sets, that have not been previously tackled via optimization techniques. 

As a result, we can construct classification trees 
that achieve $2.4\%$ higher prediction accuracy on average  compared to the 
approach of \citet{BertDunn17}, in shorter computation time (10 minutes). Although our approach is slightly outperformed by the method proposed by \citet{VerZha19} by $0.3\%$ on average, it remains competitive in testing results.
At the same time, our approach is capable of handling much larger datasets than all state-of-the-art ILP-based methods in the literature.

This paper is organized as follows. Section \ref{literature} revises  the existing literature and discusses the state-of-art algorithms in constructing decision trees. Our basic notation and important concepts related to decision trees are introduced in Section \ref{Sec:Preliminaries}. Sections \ref{Sec:Const}  and \ref{sec:CG procedure} present  the mathematical models and  our solution approach. Section \ref{Sec:CompExpr} reports the experimental results  obtained with our method and compares them to  recent results in the literature. Finally, our conclusions and further research directions are discussed in Section \ref{Sec:conclusion}.

\section{Related work}\label{literature}

\paragraph{Decision trees}
Finding optimal decision trees is known to be NP-hard (\citet{HyaRi76}). This led to the development of  heuristics that run in short time and output reasonably good solutions. An important limitation in constructing decision trees is that the decision splits at internal nodes do not contain any information regarding the quality of the solution, e.g. partial solution or lower bounds on the objective. This results in lack of guidance for constructive algorithms \cite{Breiman84}. To alleviate  this shortcoming,  greedy algorithms use goodness measures for making the (local) split decisions. The most common measures are Gini Index, used  by CART (\citet{Breiman84}), and Information Gain, used by ID3 (\citet{Quinlan86}). In order to increase the generalization power of a decision tree, a pruning post-processing step is usually applied after a greedy construction. 

\citet{Norton89} proposed adding a look-ahead procedure to the greedy heuristics, however no significant improvements are reported \cite{Murthy95}. Other optimization techniques used in the literature to find decision trees are  integer linear programming (shortly ILP), dynamic programming \cite{Payne77}, and stochastic gradient descent based methods \cite{Norouzi15}.

\paragraph{ILP-based approaches for constructing decision trees}
Several ILP approaches have been recently proposed in the literature.  
\citet{BertDunn17} study constructing optimal classification trees with both univariate and multivariate decision splits. The authors do not assume a fixed tree topology, but control the topology through a tuned regularization parameter  in the objective.  As the magnitude of this  parameter increases, more leaf nodes may have no samples routed to them, resulting in  shallower trees.  An improvement of 1-2\% in accuracy w.r.t CART is obtained for out-of-sample data for univariate test and an improvement of 3-5\% for multivariate tests.  

By exploiting the discrete nature of the data, \citet{Gunluk18} propose an efficient MILP formulation for the problem of constructing classification trees for data with categorical  features. At each node, decisions can be taken based on a subset of  features (combinatorial checks).  The number of integer variables in the obtained MILP is independent of the size of the training data.  Besides the class estimations to the leaf nodes, a fixed tree topology  is given as input to the ILP model. Four candidate  topologies are considered, from which one is eventually chosen after  a cross validation.   Numerical experiments indicate that, when classification can be achieved via a small interpretable tree, their algorithm outperforms CART regarding accuracy.

In another recent study, \citet{Dash18} propose an ILP model for learning Boolean decision rules  in disjunctive normal form (DNF, OR-of-ANDs, equivalent to decision rule sets) or conjunctive normal form (CNF, AND-of-ORs) as an interpretable model for classification.  The proposed ILP takes into account the trade-off between accuracy and the simplicity  of the chosen rules and is solved via the column generation method.  The authors formulate an approximate pricing problem by randomly selecting a limited number of features and data instances.  Computational experiments show that this Columns Generation (CG) procedure  is highly competitive to other state-of-the-art algorithms. 
 
Our ILP builds on the ideas in  \citet{VerZha17}, where an efficient encoding is proposed for constructing both classification and regression (binary) trees of univariate splits of depth $k$. As a result,  the number of decision variables  in their ILP  is reduced to $O(|R|k)$, compared to  $O(|R|2^{k})$ variables used in  \citet{BertDunn17}, where $|R|$ is the number of rows in the considered dataset.  Preliminary results indicate that the method used in \citet{VerZha17} obtains good results on trees  up to depth 5 and smaller data sets of size up to 1000. Recently, \citet{VerZha19} 
formulate the optimal decision tree learning problem of a given depth as a binary linear program. The method is called BinOCT. They show that BinOCT outperforms the existing approaches of \cite{VerZha17, BertDunn17} on a variety of data sets in terms of accuracy and computation time.  Although BinOCT speeds up learning, the largest dataset that the authors show the method can achieve reasonable results within 10 minutes is still rather small, containing no more than 5000 rows. 

\citet{VerZha17} demonstrate the advantages of using a MILP model for learning decision trees, by showing that discrimination-aware classification trees and trees with minimised  false positive/negative errors can be easily obtained by modifying the objective function of the MILP model. However, the state-of-the-art MILP based approaches, i.e. \cite{VerZha17, VerZha19, BertDunn17}, are only shown to be effective on small datasets with less than 5000 rows. We fill this gap in this paper, and use  \citet{VerZha19} and \citet{BertDunn17} as the main reference for benchmarking our method (see Section \ref{Sec:CompExpr}).

\paragraph{Column Generation} Column generation (CG) is a widely used technique for solving large scale Linear Programs, usually with exponential number of variables. Initially proposed by  \citet{Ford58} for a  multicommodity network flow problem, it has been widely used in the last decades for a variety of ILP problems such as  cutting stock \citet{Gilmore61}, vehicle routing \citet{Desrosiers84},  airline crew scheduling \citet{Bornd06}, employee scheduling \citet{Al-Yakoob08}, workforce assignments in telecommunication \citet{Firat16}, and many others. We will use CG in this paper to find a good solution for the proposed ILP.

\section{Preliminaries}\label{Sec:Preliminaries} 

 In this section we describe the basic concepts of our work and introduce the necessary notation.
 
The goal of a classification problem is to partition data instances in classes labeled by a set of targets $T$.  We assume that a data instance is given as a row $r$ containing a value $v_f^r$ for every feature $f$ in the set $F$. Without loss of generality, we only consider numerical features.  Ordinal and categorical features can be transformed into numerical ones, by using natural numbers for ordinal and one hot encoding for categorical data. 

Let $BT$ be a full binary tree of depth $k$ and let $N_{int}$ and $N_{lf}$  be the internal and leaf nodes of $BT$ respectively. A path from root to a leaf $l\in N_{lf}$ will be denoted by $p_{BT}(l)=(n(0),..., n(k-1),l)$, where $n(h)$ is a node at level $h$ in $BT$. Let $S$ be a set of \textit{univariate decision splits}, that is,  pairs  $(f,v_f\leq \mu )$, with $f\in F$, $v_f \in \mathbb{R}$ and $\mu \in \mathbb{R}$. $v_f$ is called the \textit{value of feature f}, while $\mu$ is called a \textit{threshold}.  

\begin{definition}A binary decision tree $DT$ of depth $k$ can be viewed as a triple $DT=(BT, s, \hat{t})$, where
$s:N_{int}\mapsto S$ assigns to each node a split, and  $\hat{t}:N_{lf}\mapsto T$ assigns a target to each leaf. \end{definition}

Consider a given decision tree $DT$. When classifying a data instance (row) $r$ with the help of $DT$, if  the split at node $j$ is $s(j)=(v_f\leq \mu)$ and  the value of feature $f$ for row $r$ satisfies $(v_f^r\leq \mu)$, the data row will be directed to the left child of $j$; if $v_f^r> \mu$, it will be directed to the right child. In this way, in the classification process, each data row is directed from the root node $n(0)$ to one of the leaves in $N_{lf}$ through a sequence of splits. If a data row $r$ ends in leaf $l$, $r$ belongs to the class labeled by target $\hat{t}(l)$. 

Consider now a row $r$ for which the target $t_r$ is known. We say that row $r$ is \textit{correctly predicted} by the decision tree $DT$ if the following two conditions are satisfied:\\
(i) $r$ is directed to $l$ through the splits of $DT$\\
(ii) $t_r=\hat{t}(l)$, where $\hat{t}(l)$ is the target at leaf node $l$. 

In the classification tree construction problem, shortly $k$-CTCP, considered in this paper, we are given a data set $R$ with known targets, a binary tree $BT$ of depth $k$, a set of decision splits $S_j$ for each internal node $j\in N_{int}$ and a set of targets $T$. The sets $S_j$ are usually determined based on the features of the rows in $R$.

\begin{definition}\label{def:decpath}
Let $p_{BT}(l)=(n(0),..., n(k-1),l)$ be the path from the root node to a leaf $l$ in $BT$. A sequence $p(l)=(s(n(0)),..., s(n(k-1)), \hat{t}(l))$ is called a \textit{decision path} from root to leaf $l$, where $s(n(h))$ is the split associated to the node on level $h$ on $p_{BT}(l)=(n(0),...., n(k-1),l)$. Decision path $p(l)$ should satisfy the following three conditions: \\
(i) $s(n(h))\in S_{n(h)}$\\
(ii) $s(n(h))\neq s(n(h'))$, for $h,h'\in \{0,..., k-1\}$, $h\neq h'$. \\
(iii) $\hat{t}(l)\in T$ 

\end{definition}

For the easy of explanation, we will sometimes see $p_{BT}(l)$ as a collection of nodes. For two paths $p_{BT}(l)$ and $p_{BT}(l')$ in $BT$, $p_{BT}(l)\cap p_{BT}(l')$ will represent the intersection of the set of nodes on each of the paths. A similar convention will be used for decision paths.

Finally, for a given data set $R$ with known targets and decision tree $DT$, we denote the number of correct (true)  predictions at leaf $l$ by $\mbox{CP}(p(l))$,

The goal of the $k$-CTCP is to find the assignments  $s:N_{int}\mapsto S$ and $\hat{t}:N_{lf}\mapsto T$ such that $s(j)\in S_j$ and $\sum_{l\in N_{lf}}CP(p(l))$ is maximized. 

Observe that for a given assignment $s$, the  optimal prediction target class at $l$ is in $Argmax_t \{|\{r\in R^l(s): t_r=t\}|\}$, where $R^l(s)$ is the set of rows directed to leaf $l$  under the split assignment $s$. Hence, for each assignment $s$,  an  optimal set of targets can be easily calculated.  Moreover, as the tree is constructed based on the set $R$, we can assume w.l.o.g. that $ \mu \in \{v^{r}_{f}|r \in R\})$, for every threshold $\mu$. 

Let $S=\{(f,v_f^r\leq \mu), r\in R \text{ and }\mu \in \{v^{r}_{f}|r \in R\}\}$ be the set of all possible splits. We will call the ($k$-CTCP) problem with $S_j=S$, for each $j\in N_{int}$ the \textit{full information ($k$-CTCP) problem}. If there exists an internal node $j \in N_{int}$ such that $S_{j} \subset S$, then the problem is said to be \textit{restricted}. 

The problem definition above can be implemented as an integer program, by assigning splits to nodes, however, at the expense of a large computational time (see \cite{BertDunn17}).  In this paper, we propose an alternative formulation based on assigning \emph{decision paths} to leaves, which allows the use of Column Generation technique. We will see in subsequent sections that this formulation allows us to apply several acceleration procedures to keep the computation time short, enabling the construction of decision trees for larger data sets. 

Note that while there is a unique path from the root to a leaf $l$  in $BT$, there are many possible  decision paths from the root to $l$, depending on the splits assigned to the internal  nodes.  For $l\in N_{lf}$, let $DP_l$ be the set of decision paths from the root to leaf $l$. 

We say that two decision paths $p(l)$ and $p'(l')$ \textit{agree} with each other, if $s(j)=s'(j)$ for each $j\in p_{BT}(l) \cap p_{BT}(l')$. In the following we give the formal definition of the $k$-CTCP.

\begin{quote}
	{\sc Problem:} {\sc Depth-$k$  Classification Tree Construction Problem ($k$-CTCP)  }
	\\[1.0ex]
	{\sc Instance:} A set of rows $R$ with feature F and target classes T, a binary tree $BT$ of depth $k$, a set of decision paths $DP_l$, $l\in N_{lf}$ of univariate decision splits.
	\\[1.0ex]
	
	{\sc Question:}
Find an assignment $\tilde{p}:N_{lf} \mapsto \{DP_l: l\in N_{lf}\}$ such that $\tilde{p}(l)\in DP_l$, $\tilde{p}(l)$ and $\tilde{p}(l')$ agree with each other for all $l,l' \in N_{lf}$, and $\sum_{l\in N_{lf}}CP(\tilde{p}(l))$ is maximized.
\end{quote}

\section{Column generation based solution procedure}\label{Sec:Const}

Several recent papers propose MILP models for constructing univariate classification trees (\citet{BertDunn17} , \citet{Gunluk18}, \citet{Dash18}, \citet{VerZha17}).  However, these methods only report solutions for  data instances of less than 10000 rows within the predefined time limits (10 min-2 hours). Designing a MIP- based method for tackling large datasets is still a challenge. 

In this paper, we reformulate the aforementioned tree learning problem to employ Column Generation (CG) technique aiming to solve big data instances in reasonable times. Employing decomposition techniques such as CG is a commonly used technique for finding solutions to large MILP models. As in many CG formulations,  the decision variables are exponentially many, however the number of constraints is small and therefore only a relatively small number of all columns is expected to be  enumerated to solve the LP model optimally. 

Our reformulation is based on the decision paths defined in the previous section . We formulate the master ILP of the ($k$-CTCP) in  Section \ref{sec:MasterILP}.  As we will see in Section \ref{Sec:CompExpr}, this reformulation is quite strong, leading to an integer solution in many instances.

As it is usual in a CG approach, the master problem and the pricing problem are solved iteratively, where the former passes to the latter the dual variables in order to find promising columns (here decision paths), i.e. having positive reduced costs. The decision variables corresponding to the promising decision paths are added to the master LP model to improve the objective. The optimality of the master model is proven when no path with positive reduced cost exists. We refer to \citet{Desrosiers05} for more details of CG technique.

We derive the corresponding pricing subproblem in Section \ref{sec:pricing}. In Section \ref{sec:pricing} we show that a brute force algorithm can solve the pricing problem in polynomial time for a tree of depth $k$. However, the high running time of this algorithm makes it unsuitable for use  in a CG procedure. Instead, we solve the pricing problem by a  randomized heuristic  and we resort to solving a MILP formulation of the pricing problem  when needed. In Section \ref{sec:integersol} we explain how the final integer solutions are obtained after performing the CG procedure. 

\subsection{Master problem formulation}\label{sec:MasterILP}
The master model chooses a collection of agreeing decision paths that forms a feasible decision tree, as defined in Section \ref{Sec:Preliminaries}. Table \ref{Table:MasterLP} lists the sets, parameters, and decision variables of the master ILP model.
\begin{center}
	\begin{table}[!!ht]
		\footnotesize
		\caption{Sets, parameters, and decision variables for the master model} \label{Table:MasterLP}
		\begin{tabular}{ll }
			\hline
			\multicolumn{2}{l}{\textbf{\emph{Sets}}}
			\\
			$R$ & set of rows in data file, indexed by $r \in R$. \\
			$F$ & set of features in data file, indexed by $f \in F$. \\
			$N_{lf}, N_{int}$ & leaf and internal (non-leaf) nodes in the decision tree, indexed by $l \in N_{lf}, j \in N_{int}$. \\
			$DP_l$ & set of decision paths ending in leaf $l$, indexed by $p \in DP_l$. \\
						$DP_l(j)$ & subset of paths in $DP_l$, such that $j\in p_{BT}(l)$.  \\
			$R^{l}(p) $ & subset of rows directed to leaf $l$ through path $p$. \\
			$S_j$ &  set of decision splits for node  $j$.\\
			\multicolumn{2}{l}{\textbf{\emph{Parameters}}}
			\\
			
			$k$& the depth of the decision tree, levels are indexed by $h = 0,\dots,k-1$. \\
			$\mbox{CP}(p)$& number of correct predictions/true positives for  path $p$. \\
			
			\multicolumn{2}{l}{\textbf{\emph{Decision Variables}}}
			\\
			
			$x_{p}$& binary variable indicating that path $p \in DP_l$ is assigned to leaf $l \in N_{lf}$.\\
			
			$\rho_{j,a}$& binary variable indicating that split $a \in S_j$ has been assigned to node $j\in N_{int}$.\\

			\hline
		\end{tabular}
	\end{table}
\end{center}
The following lines present the master ILP formulation for the ($k$-CTCP). The formulation holds for both full information and restricted versions.
\begin{equation}\label{eq:obj} \mbox{ Maximize } \sum_{l \in  N_{lf} }\sum_{p \in DP_l} \mbox{CP}(p)x_{p} \end{equation}
\begin{center} \mbox{ subject to } \end{center}

\begin{equation}\label{eq:1path} \sum_{p \in DP_{l}}x_{p}= 1, \quad l \in N_{lf}  \end{equation}

\begin{equation}\label{eq:1row} \sum_{l\in N_{lf}}\sum_{p\in DP_l:r \in R^{l }(p)} x_{p}= 1, \quad r \in R  \end{equation}

\begin{equation}\label{eq:consist} \sum_{\substack{p \in DP_l(j):\\ s(j)=a }} x_{p}= \rho_{j,a}, \quad l \in N_{lf}, j\in p_{BT}(l)\cap N_{int} \text{ and } a \in S_j  \end{equation}

\begin{equation}\label{eq:integrality_x} x_{p} \in\{0,1\},\quad  \ p \in DP_l, \;l \in N_{lf}  \end{equation}
\begin{equation}\label{eq:integrality_rho}\rho_{j,a}\in \{0,1\}\quad j\in N_{int}, a \in S_j  \end{equation}

The objective function (\ref{eq:obj}) maximizes the number of rows correctly predicted, that is, the  accuracy of the decision tree. Constraint (\ref{eq:1path}) imposes  that exactly one path has to be selected for each leaf. Constraint (\ref{eq:1row}) ensures that each row  is directed to  exactly one leaf. Constraint (\ref{eq:consist}) ensures that all chosen paths agree with each other which means that all decision paths passing a particular internal node have the same decision split at that node. 

In order to employ the CG technique, we work with the relaxation of the master ILP model, in which constraints (\ref{eq:integrality_x}) are relaxed as $0 \leq x_p\leq 1$. Since constraints (\ref{eq:1path}) guarantee that every variable $x_p$ is bounded by $1$ from above, constraints (\ref{eq:integrality_x}) can be relaxed to $x_p\geq 0$.  Note that there is no need to impose any bounds on   $\rho_a$ , as  $\rho_a \geq 0 $ follows from  (\ref{eq:consist}) and the non-negativity of $x_p$, while   $\rho_{ja} \leq 1 $ follows from  the fact that   the sum in the left hand side in (\ref{eq:consist})  is bounded by 1, as a consequence of (\ref{eq:1path}).

\subsection{Pricing subproblem}\label{sec:pricing}

In this section we formulate the pricing problem and show that it  can be solved by a polynomial time algorithm. However, the running time of this algorithm is too high for the use in a CG approach. We then formulate the pricing problem as a MILP, which will be used in the heuristic proposed in subsequent section.

We associate the dual variables  $\alpha_{l}, \beta_{r}$, and $\gamma_{l,j,a}$ with constraints (\ref{eq:1path})-(\ref{eq:consist}) respectively. Given that the number of paths in the sets $DP_{l}$, $l\in N_{lf}$  are exponentially many, these sets will not be fully enumerated. Instead, we only find the paths in $DP_l$ that are promising for increasing the objective value. A path is said to be promising if it has a positive associated reduced cost which is defined as

\begin{equation}\label{redcost}
\overline{\mbox{CP}}(p) = \mbox{CP}(p) - \alpha_{l} - \sum_{\substack{ (j,s(j)): j \in p_{BT}(l)\cap N_{int}}} \gamma_{{l,j,s(j)}} - \sum_{r \in R^l(p)} \beta_r,\quad p \in DP_{l}, \; l \in N_{lf}
\end{equation}

The pricing problem can be defined as finding  $ z_{Pr}^{*} = \max_{p}\{\overline{\mbox{CP}}(p)\}$. 

 It can be easily seen that  the problem of finding $z_{Pr}^{*}$ can be decomposed into solving $|L||T|$ optimization sub-problems, each of them corresponding to a pair $(l,t)\in L\times T$. More precisely, 
 $$z_{Pr}^{*} =\max_{(l,t)\in L\times T} z^*_{Pr}(l,t),$$
 where $z^*_{Pr}(l,t)=\max_{p\in D_l^t}\{\overline{\mbox{CP}}(p)\},$
and  $D_l^t$ is the set of paths in $D_l$ with target $t$ associated to leaf $l$. 

\subsubsection{Complexity } \begin{proposition} \label{prop:pricing} For any pair $(l,t) \in L\times T$, $z^*_{Pr}(l,t)$ can be found in polynomial time.
\end{proposition}

\begin{proof}
The number of decision paths in $D_l^t$ is $\prod_{j\in p_{BT}(l)}|S_j|=O(|R|^{k}|F|^{k})$,as $|S_j| \leq |R||F|$. For a given decision path $p\in DP_l^t$, computing $R_l$ takes $k|R|$ operations, while computing the reduced costs takes $|R|$ operations. Hence, for fixed $k$, the running time of the algorithm is $O((k+1)|R|^{k+1}|F|^{k})$, which for a fixed $k$ is polynomial in the instance size.  
\end{proof}

Proposition \ref{prop:pricing} implies that $z_{Pr}^{*}$ can be found in polynomial time. However, the running time of the polynomial algorithm is high, which renders it inadequate for a repeated use in a CG procedure. 

In Appendix \ref{app:Complexity}  we show that for an arbitrary $k$, the decision path construction problem  is NP-hard. 

\subsubsection{MILP formulation of the pricing problem}\label{MIP-pricing} 

Next we reformulate the problem of finding $z^*_{Pr}(l,t)$  as a MILP.  The necessary notation is presented in Table \ref{Table:Pricing}.\\

By noting that  for $p\in D_l^t$, the number of correct prediction can be rewritten as
\begin{equation}CP(p) = \sum_{r \in R_{t}} y_{r}.  \end{equation}
the problem of finding $z_{Pr}^{*}(l,t)$ becomes:
\begin{center}
	\begin{table}[!!ht]
		\footnotesize
		\caption{Sets, parameters, and decision variables for the pricing} \label{Table:Pricing}
		\begin{tabular}{ll }
			\hline
			\multicolumn{2}{l}{\textbf{\emph{Sets}}}\\
			
			$R_{t}$ & set of rows in data file with target $t$, indexed by $r \in R_t$. \\
			$F$ & set of features in data file, indexed by $f \in F$. \\
			$S_j$ &  set of decision splits for node  $j$.\\
            $LC(l)$& set of nodes in $N_{int}$ with left child in $p_{BT}(l)$\\
	           $RC(l)$& set of nodes in $N_{int}$ with right child in $p_{BT}(l)$\\
	           $T(r)$& set of splits for which row $r$ returns a TRUE:  $ \{a=(f,v_f \leq \mu) \in S_{RC(l)}\cup S_{LC(l)} : v_f^r \leq \mu\}$ \\
	           $F(r)$& set of splits for which row $r$ returns a FALSE:  $ \{a=(f,v_f \leq \mu) \in S_{RC(l)}\cup S_{LC(l)} : v_f^r > \mu\}$ \\
			
			\multicolumn{2}{l}{\textbf{\emph{Parameters}}}\\
			$v^{f}_r$ & value of feature $f$ in row $r$. \\
			
			\multicolumn{2}{l}{\textbf{\emph{Decision Variables}}}\\
			$y_{r}$&binary variable indicating that row $r \in R$ reaches leaf $l$. \\
			$u_{j,a}$&binary variable indicating that split $a\in S_j$ is assigned to node $j\in RC(l)\cup LC(l)$*. \\
			
			\hline
	\multicolumn{2}{l}{* For path $p_{BT}(l)$ we have $RC(l)\cup LC(l) = p_{BT}(l) \cap N_{int}$.}\\
		
		\end{tabular}
	\end{table}
\end{center}
\begin{equation}\label{eq:objpricing} DPP(l,t): \quad\quad \mbox{ Max } \sum_{r \in R_t} y_{r} - \alpha_{l} - \sum_{j \in p_{BT}(l) }\sum_{a \in S_j } \gamma_{l,j,a} u_{j,a} - \sum_{r \in R} \beta_{r} y_{r} \end{equation}
\begin{center}\mbox{ subject to }\end{center} 

\begin{equation}\label{eq:1h} 
\sum_{\substack{a \in S(j) }} u_{j,a} = 1, \quad j \in RC(l)\cup LC(l)
\end{equation}

\begin{equation}\label{eq:leftp} 
y_{r} \leq \sum_{a \in S_j \cap T(r)} u_{j,a}, \quad   j \in LC(l),  r \in R
\end{equation}

\begin{equation}\label{eq:rightp} 
y_{r} \leq \sum_{a \in S_j\cap F(r)} u_{j,a}, \quad  j \in RC(l),   r \in R
\end{equation}

\begin{equation}\label{eq:allp} 
\sum_{j \in LC(l)}\sum_{ a\in S_j\cap T(r)}   u_{j,a}  + \sum_{j \in RC(l)}\sum_{ a\in S_j\cap F(r)}  u_{j,a} - (k-1) \leq y_{r} \quad r \in R
\end{equation}

\begin{equation}\label{distinct} 
\sum_{\substack{ j \in p_{BT}(l) }} u_{ja} \leq 1, \quad a \in \cup_{j\in RC(l)\cup LC(l)}S_j
\end{equation}

\begin{equation}\label{y_indicator}
y_{r}\in  \{0,1\}, \quad r \in R
\end{equation}

\begin{equation}\label{prbounds2}
u_{j,a} \in \{0,1\}, \quad j\in RC(l)\cup LC(l), a\in S_j
\end{equation}

Objective (\ref{eq:objpricing}) aims to maximize the reduced cost associated to a feasible decision path. Constraint  (\ref{eq:1h}) ensures  that exactly one decision split has to be performed at each level. Constraints (\ref{eq:leftp}), (\ref{eq:rightp}) and (\ref{eq:allp}) take care  that the rows directed through the nodes of the path are consistent with the decision splits. Finally, constraints (\ref{distinct}) enforce that the splits performed  at  internal nodes are distinct. 
Note that constraints (\ref{eq:1h}) and  (\ref{eq:rightp})  ensure that $y_r\leq 1$. Moreover, all the right hand side bounds in (\ref{eq:leftp}) and (\ref{eq:rightp}) are binary. Hence, since the problem above is a maximization problem and $u_{j,a}\in \{0,1\}$, constraints (\ref{y_indicator}) can be relaxed to  $y_r\geq 0$.
\newline
The outcome of each optimization problem $DPP(l,t)$ is a decision path $(s(0),...,s(k-1),t)$. 

\subsubsection{Pricing heuristic: Constructing randomized decision paths} 
\label{subsec:heuristic}

The goal of the pricing heuristic is to  quickly find decision paths with positive reduced costs. We keep a set $C$ of feasible columns (decision paths) of a fixed-size. At each CG iteration, $C$ is updated by randomly constructing a set of columns, estimated to be promising, and by removing a subset columns with low reduced cost values. 

 Let $N_{lf}'$ be the set of leaves for which decision paths with positive reduced cost were found in the previous iterations ($N_{lf}'=N_{lf}$ in the first iteration). We select uniformly at random  $n_{l}$ leaves from $N_{lf}'$ (a leaf can be selected several times). For each selected leaf $l$, we construct a decision path by selecting uniformly at random a split from $S_j$, for each $j\in p_{BT}(l)\cap N_{int}$ by ensuring the selected splits are distinct.
 
Among the columns (decision paths) at hand, the $n_c$  with the highest positive reduced costs are added to the master problem. If the number of columns with positive reduced cost is lower than $n_c$, then all columns with positive reduced costs are added. Finally,  columns with low reduced costs are removed to maintain the fixed size of the set $C$. 

The pricing heuristic is used as long as it delivers  columns with positive reduced costs. If no promising column is found after running the pricing heuristic a given number of times, the  algorithm switches to solving the MILP formulation of the pricing model for checking optimality of the master problem. If the MILP model  finds some promising columns, we remove all columns in set $C$ and adjust the randomized pricing heuristic such that leaves for which decision paths with high (positive) reduced cost are found, have priority to be considered. Otherwise, the master problem is solved optimally.

\section{Selecting the restricted sets of splits} \label{sec:CG procedure}
There are many types of NP-hard problems, especially of the set-partitioning type, where the CG procedure reduces considerably the computation time of the MIP formulations. For example,  CG proved to be very efficient for  vehicle routing (see e.g. \cite{SplietGabor2014}) and workforce assignments (see \cite{Firat16}). \\

However, our preliminary experiments indicated that a CG approach on the full information problem, that is, the problem with  $S_j=\big \{(f, v_f^r\leq \mu)| r\in R , \mu \in \{v^{r}_{f}|r \in R\}\big \}$, has difficulties in finding optimal solutions for the $k$-CTCP in reasonable times.
The main difference between $k$-CTCP and a set-partitioning problem is in the inter-dependency between paths. To define a decision tree, the paths need to agree with each other,  i.e., share the same splits at common internal nodes (see constraints (\ref{eq:consist})).  This set of constraints makes the decision paths highly dependent on each other, explaining why the CG algorithm is less efficient than for problems with a clear set partitioning structure.  Moreover,  to formulate these constraints,  the set of variables $\rho_{j,a} $ are needed.   When all the decision splits are considered, the number of these variables   is of magnitude  $O(|R|2^{k-1})$. The high dependency between decision paths and the extra variables  increase the complexity of the master model of the CTCP considerably. 

To speed up the algorithm, we propose to approximate the value of the $k$-CTCP problem with full information by a restricted version, in which at each node, only a subset of splits $S_j\subseteq S$ is used to define the decision tree.  As a consequence, the optimality guarantee of the trees is lost. 

To find a good restricted set of decision splits
$S_j$ at internal node $j$, we make use of the CART algorithm.  For simplicity, we will call this process \textit{threshold sampling}, despite the fact that we sample from both sets of features and thresholds. \\ 

In the threshold sampling procedure, we run the CART algorithm \cite{sklearn}  on a randomly selected large portion  of the training data, i.e. $\alpha \%$ (line 4, Algorithm \ref{thrslg}) and collect the  decision splits appearing in the obtained  tree (lines 5 and 6, Algorithm \ref{thrslg}). Note that the splits in CART are distinguished by the superscript CART ($S^{CART}$).  This procedure is repeated while a new decision split appears at root node in less than $\tau$ iterations. We then retain the splits in sets $S_j$'s, that are most frequently used at each internal node.  While  it is possible to keep all decision splits at the root node, as their number is small, we only keep a limited number of the decision splits appearing at the  internal nodes of the constructed CART trees.  More precisely, we keep at every internal node $j$, the  $q_j$ splits with the highest frequency  (line 14, Algorithm \ref{thrslg}). Additionally, we also use the CART solution generated with 100\% of the training data. Besides the splits, we also add the decision paths of that CART solution to the master problem which ensures that the obtained tree at the end of the whole CG-based procedure performs at least as well as CART regarding training accuracy. However, such a conclusion cannot be drawn regarding testing accuracy. The stopping rule is based on the observation that the split at the root and at nodes close to the root are the most  decisive in the structure of  the tree. 

\begin{algorithm}[H]
  \caption{Threshold sampling procedure}
  \label{thrslg}
  \begin{algorithmic}[1]
 
    \State 	\textbf{INPUT}: Problem instance described in Section \ref{Sec:Preliminaries},  Parameters $\tau$, $\alpha$, $q_j \in Z_{+}$. \\ Initialize: $S_j=\emptyset$, $w_{(j,a) }=0$, \; $j\in N_{int}$ $ a\in S$; and $i = 0$;
    \While{ $i  < \tau $ }
    \State Randomly select $\alpha \%$ of training data, and use CART to construct  a tree $CART_{temp}$ 
    \State $w_{(j,a)} \leftarrow w_{(j,a)}+1, \; (j,a) \in S^{CART_{temp}},$ // frequency updates
    \State $S_j \leftarrow S_j \cup \{S^{CART_{temp}}(j)\}$, $ j \in N_{int}$, // Update of the Split sets
     \If{$ S^{CART_{temp}}(root) \in S(root)$   }
    \State $i \leftarrow i+1$
    \Else{}
    \State $i \leftarrow 0$ 
    \EndIf  		
    \EndWhile 
    \State Select $100\%$ of training data, and use CART to construct  a tree $CART_{temp}^{100\%}$ 
		
		\State 	\textbf{OUTPUT}: \{The  $q_j$ splits with highest frequency $S_j$\}$\cup \{ S^{CART_{temp}^{100\%}}(j) \}$
  \end{algorithmic}
\end{algorithm}

\section{Finding an integer solution} \label{sec:integersol}

CG is a solution technique to solve LP models, hence the solution obtained when the procedure terminates is fractional in general.  In case indeed a fractional solution is delivered by the CG algorithm, we solve the master ILP model with the columns enumerated so far. In our experiments, the formulation of the master LP model proposed in Section \ref{sec:CG procedure} delivered a high ratio of integer solutions, an indication of the good quality of the LP bound.  Detailed results regarding integrality are listed by Table \ref{Table:Int} in the computational experiments section.

In our CG-based heuristic approach, we obtain an integer solution by converting the master LP model  to an ILP model. Our experiments indicate that a standard solver can find the optimal solution of these ILP models very fast. We may give several reasons for this. Firstly, we construct the decision paths of one complete CART solution that is found with 100\% of the training data during threshold sampling which provides a warm-start to the solver. Secondly, the strong constraint that all decision paths should agree in their decision splits may turn out to be helpful in finding quickly valid cuts, hence pruning the nodes of the search tree. 

As common in the literature, we have experimented by using the CG based approach in a Branch-and Price algorithm. The main difficulty we faced was in choosing good branching rules. In our trials, we were not able to find efficient branching rules, that resulted in trees of a small size for large data sets. As the gap between the lower and upper bound obtained from the CG-based procedure was already small, we chose to settle for the obtained upper bound. However, designing an efficient branch and price is a pending future research.

\section{Computational experiments}\label{Sec:CompExpr}

This section presents the  computational results obtained with our CG based heuristic approach, shortly CGH, and compares them to the results of the recently proposed ILP based classification algorithms in the literature. 

The purpose of our experiments is three-fold. First, we are interested in the performance of the CGH in improving the solution quality of its baseline method and its competitiveness with the optimal classification tree methods recently proposed in the literature. We use the training accuracy of the built classification trees to investigate this. Second, we test the generalization capability of CG on learning classification trees. We compare the testing accuracy in this evaluation. Third, and most importantly, we investigate the scalability of the CGH. To this end, we test the CG algorithm on large datasets, which to the best of our knowledge, have not been attempted by the existing ILP based approaches of learning decision trees.

\subsection{Baseline and benchmark algorithms}

The proposed CGH algorithm uses the relaxation of the master ILP formulation that is solved by decomposing it into pricing subproblems as described in Section \ref{sec:MasterILP} and the threshold sampling described in Section \ref{sec:CG procedure}. The baseline algorithm for CGH is the CART that is available in Scikit-learn, which is a machine learning tool in Python \cite{sklearn}. We ran CART with the default parameters except for the maximum depth, which was set to the corresponding depth of the tree solved by our method. We used two benchmark methods for the evaluation of CGH performance. The first one is the MILP formulation OCT, proposed by \citet{BertDunn17}. 
and the second one is BinOCT* recently proposed by \citet{VerZha19}.

\subsection{Experimentation setting}\label{subsection:Experimentation}
In the pricing heuristic, we use $|C|=500$ for the randomized decision path construction, the number of leaves in the updated procedure is $n_l=200$, and the number of the chosen columns to add to the master problem is  $n_c=100$.  
For the threshold sampling procedure \ref{thrslg}, we  use the following parameter values:  the portion of the data $ \alpha = 90\%$, the number of CART trees is 
$\tau = 300$, and the number of $q_{root} = \floor*{\frac{150}{|N_{int}|}}$, $q_{j} = \floor*{\frac{100}{|N_{int}|}} \; j \in N_{int} \setminus \{root\}$ decision splits are selected. 

\nop{
\begin{center}
\begin{table}[h]
 \centering
 \footnotesize
  \caption{Tuned hyperparameters for CART*} \label{Table:Hyperparams}
  \begin{tabular}{ll }
\hline 
\textbf{Parameter } & \textbf{Range set} \\ 
\hline
Goodness criterion & $\{\mbox{gini, entropy}\}$. \\
Minimum sample requirement & $\{0.02,0.05,0.1,0.2\}$ \\
Class weight & $\{\mbox{None, Balanced}\}$ \\
Minimum segment size at leaves & $\{\mbox{0.01, 0.05, 0.1, 0.2,1}\}$ \\
\hline
\end{tabular}
\end{table}
\end{center}
}

We tested four algorithms using 20 datasets from the UCI repository \cite{UCI}, where 14 are ``small'' datasets containing less than 10000 data rows and 6 are ``large'' ones containing over 10000 data rows. An overview of the used datasets is listed in Table \ref{tab:instances}.

 We use the same experiment settings as in \citet{VerZha19} \footnote{\url{https://github.com/SiccoVerwer/binoct}}. In this setting, a given dataset is split 50\% for training and 25\% for testing, as it was also used by \citet{BertDunn17}, and the split is done randomly five times. A timelimit of 10 minutes is used for solving each instance and the average performance of five experiments for each dataset is reported. Eventually, we take the same training and testing instances used by BinOCT* to run our experiments. We used the algorithms to construct classification trees of depths 2, 3, and 4 and compare their performance in terms of training and testing accuracy. Since the precise setting details of \citet{BertDunn17} were not available, it is not possible to compare CGH and BinOCT* to OCT training accuracies on training instances. Therefore, we decided to make this comparison only on testing instances.

All experiments were conducted on a Windows 10 OS, with 16GB of RAM and an Intel(R) Core(TM) i7-7700HQ CPU @ 2.80 GHz. The code is written in Python 2.7 and the solver used to solve the linear and integer programs is CPLEX V.12.7.1 \cite{CPLEX} with default parameters.

\subsection{Results on small instances} \label{sec:smallres}

In this section we present the results of our computational experiments in small data instances. Figure \ref{fig:AccuSmall} gives an overview of the average prediction accuracy of all benchmark algorithms for increasing 2, 3, and 4. The results are shown for two cases; training data and testing data.
For the training instances, CGH is evaluated using the results of CART and BinOCT* methods as benchmarks. 

On the training instances, all methods have increasing performance with increasing depth values. Both CGH and BinOCT* exceed 90\% average accuracy for depth 4.
On these data, CGH improves on CART by 0.9\%, 2.7\%  and 1.3\% on average, for depths 2, 3, and 4 respectively. However, CGH is outperformed by BinOCT* by 1.3\%, 0.8\% and 1.2\% on average. In overall 42 results, BinOCT*, CGH, and CART have 36, 13, and 5 wins respectively, including ties. For further details for instance-wise results we refer to Table \ref{tab:smalldatatr} in the appendix.

On testing data, all the methods tested have comparable performance. CGH improves the accuracy of CART on trees of depth 2, 3, and 4 by 0.4\%, 1.4\% and 0.7\% respectively.  The average relative improvement in accuracy  of  CGH upon OCT  are 0.3\%, 2.3 \%, and 4.6\% for depths 2, 3, and 4. On trees of depths 2 and 4, CGH is outperformed by BINOCT* by an average accuracy of 1.4\% and 0.2\%, while on trees of depth 3 CGH outperforms BINOCT* by  0.7\%. In overall 42 results, BinOCT*, OCT, CGH, and CART have respectively 16, 11, 10 and 6 wins, including ties. The accuracy of  CGH and BINOCT* for depth 4 is remarkably lower on testing compared to training data, which can be an indication that these two algorithms overfit. For further details for instance-wise results we refer to Tables \ref{tab:test2}
-\ref{tab:test4} in the appendix.

\begin{figure}[H]

\centering

\includegraphics[width=1.0\textwidth]{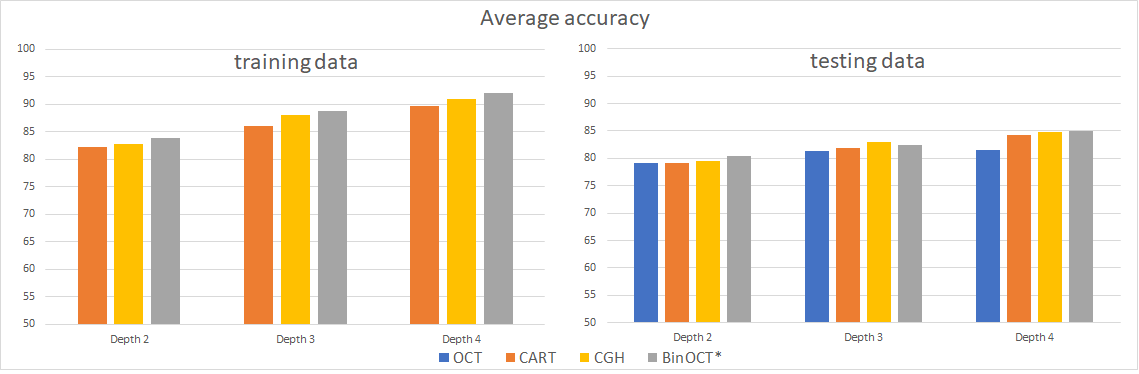}
\caption{Average accuracy (\%) of benchmark algorithms in small data sets}
\label{fig:AccuSmall}

\end{figure}

\subsection{Results on big instances}\label{sec:bigresults}

The MILP-based formulations in the existing literature (e.g., \citet{BertDunn17}, \citet{VerZha17}, \citet{VerZha19}) failed to handle datasets with more than 10000 rows within their predefined time limits. Therefore, for large datasets, we can only compare the results of our CG algorithm with CART and a tuned version of CART, shorly CART*. A more detailed explanation on tuning hyperparameters of CART can be found in \ref{App:CART}.

\begin{table}[H]
\begin{center}
\begin{tabular}{c|c|c|c|c|c|c|c|c|c}
Datasets & CART & CART* & \textit{CGH} & CART & CART* & \textit{CGH} & CART & CART* & \textit{CGH} \\
\hline

 &\multicolumn{3}{l}{ $k= 2$} &  \multicolumn{3}{l}{ $k=3$} &  \multicolumn{3}{l}{ $k=4$}   \\
\hline
Magic4 & 78.4 & 78.4 & \textbf{79.1} & 79.1 & 79.2  &\textbf{80.1} & \textbf{81.5} & \textbf{81.5} & \textbf{81.5}  \\
Default credit &\textbf{ 82.3} & \textbf{82.3} & \textbf{82.3} & \textbf{82.3} & 82.2  & \textbf{82.3} & \textbf{82.3} & 82.2 & \textbf{82.3}  \\
HTRU\_2 & \textbf{97.8} & \textbf{97.8} & \textbf{97.8} & \textbf{97.9} & 97.8  & \textbf{97.9} & \textbf{98.0} & 97.7 & \textbf{98.0}  \\
Letter recognition & 12.5 & \textbf{12.7} & \textbf{12.7} & 17.7 & \textbf{23.3}& 18.6 & 24.8 & \textbf{35.4} & 27.0  \\
Statlog shuttle & \textbf{93.7} & \textbf{93.7} & \textbf{93.7} & 99.6 & 99.5 & \textbf{99.7} & \textbf{99.8} & 99.6 & \textbf{99.8}  \\
Hand-posture & \textbf{56.4} & \textbf{56.4} &\textbf{56.4} & 62.5 & 62.4  & \textbf{62.8} & 69.0 & 69.0 & \textbf{69.1}\\
\end{tabular}
\caption{Testing accuracy of the classification trees for big data instances.}
\label{tab:test-big}
\end{center}
\end{table}

Table \ref{tab:test-big} lists the testing accuracy values of all algorithms for depths 2,3, and 4. The highest accuracy in an instance for every depth value is indicated in bold which is called a win. As seen in Table \ref{tab:test-big}, when the tree is small (i.e., depth 2), CGH slightly outperforms CART and CART*. When learning trees of depths 3 and 4, although CGH is outperformed by CART* in terms of the average accuracy over all big datasets, CGH has more wins than CART*.  

Table \ref{tab:test-big} also shows that CG has the highest number of wins for not only the small but also the large datasets.  
CGH is better than or equal to CART and CART* on 5 out of 6 datasets on constructing trees of depth 3, and its performance on one particular dataset (Letter Recognition, See Table \ref{tab:test-big}) leads to an lower average accuracy than CART*.

Our experiments demonstrate that the proposed CGH method is capable of improving the CART solutions both in small and big data instances.

Despite the large size of the problem, CGH improves more than half of the CART solutions, although the improvements, averaging 0.34\%, are not as significant as on small datasets. 

On two datasets, Default credit and HTRU\_2, CGH could not find improved solutions compared to CART. This may be caused by the structure of the data, that is, these two cases are rather easy as CART can give very good classification results already (more than 82\% for Default credit and more than 97\% for HTRU\_2). 

Compared with CART*, we note small improvements (around 0.3\%) in most of the instances. For the case \textit{Letter recognition} with 26 classes, CART* appears to be much better than CGH on predicting the right classes.

\paragraph{Computation time} An important aspect in constructing trees of a certain depth is the computation time needed. In our experiments, CART only took approximately 0.1s to learn a tree. For OCT, the time limit was set to 30 minutes or to 2 hours depending on the difficulty of the problem (see \citet{BertDunn17}). For BinOCT*, the time limit was 10 minutes. 

On small datasets, the proposed CGH algorithm terminates as soon as one of the following stopping criteria is met: (1) the optimal solution of the master problem has been reached; and (2) a time limit of 10 minutes has been reached. On all the 14 small datasets, CGH terminated before the time limit of 10 minutes is reached. On big datasets, solving the pricing MILPs was extremely time-consuming (several minutes), and therefore an additional stopping criterion is added. If no promising column is found after running the pricing heuristic a given number of times as described in section \ref{subsec:heuristic}, the algorithm stops and does not solve the MILP pricings.

\begin{figure}[H]
\begin{center}

\includegraphics[width=0.8\textwidth]{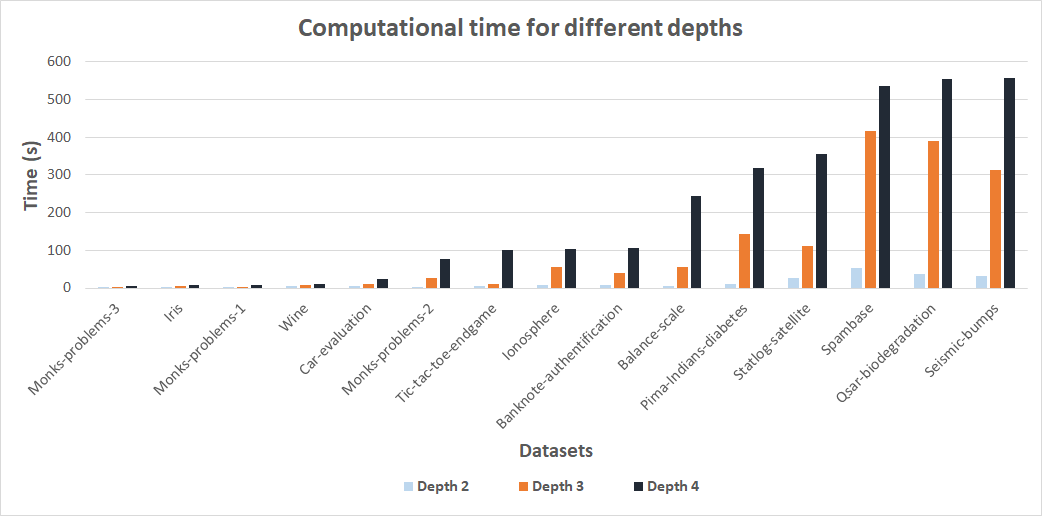}
\caption{Computational time of CGH on small datasets.}
\label{fig:comptime}
\end{center}
\end{figure}

Figure \ref{fig:comptime} shows the required computational time for constructing each tree.
As expected, the algorithm needs more time when the size of the problem (depth, rows, features) grows. Nevertheless, CGH was able to solve the master problem optimally for all instances within 10 minutes. For the smallest instances, CGH terminated in only a few seconds. 

\paragraph{Integrality of LP solutions} An important property of an ILP formulation is its strength, i.e. the quality of its bound, when its relaxation is solved. Table \ref{Table:Int} shows that all LP solutions for depth-2 are integer. In depth 3, 77\% of LP solutions are integers, and in depth 4 we obtain 57\%.

\begin{center}
	\begin{table}[h]
		\centering
		\footnotesize
		\caption{Ratio of integer solutions after CG terminates for small data instances.} \label{Table:Int}
		\begin{tabular}{lll}
			\hline 
			\textbf{Depth } & \textbf{Ratio of Integer solutions} & Percentage (\%)  \\ 
			\hline
			2 & 70/70 & 100 \\
			3 & 54/70 & 77 \\
			4 & 40/70 & 57 \\
			\hline
		\end{tabular}
	\end{table}
\end{center}

The following main conclusions can be drawn based on our computational experiments:

\begin{itemize}
    \item \textit{Competitiveness:} The CGH method is competitive to the state-of-the-are ILP-based methods in the recent literature. In training data of small instances, CGH is outperformed only by BinOCT*. In  testing data instances, CGH outperforms OCT and achieves a comparable performance to BinOCT*.
    
 \item \textit{Strength of the master ILP model}: In small data instances, 78\% of the solutions of the LP relaxation of our master model are integer. This can be seen as an indication for the strength of our master ILP formulation.  
 
 \item \textit{Big data instances:} 
  CGH fairly improves the CART solutions and is competitive to the tuned version of CART.

\end{itemize}

\section{Discussion and conclusion}\label{Sec:conclusion}

Constructing optimal classification decision trees is one of the central problems in Machine Learning. The goal of our study was to design a quick heuristic to find solutions for large instances, while maintaining high solution quality for smaller ones. To achieve this, we propose a novel ILP formulation for constructing classification trees, based on defining decision variables related to decision paths rather then decision splits, as previously done in the literature. A decision path is a sequence of decision splits from root to one of the leaf nodes. Due to the exponential number of decision paths, our master formulation is solved by employing the Column Generation (CG) technique. If the CG does not finish with an integer solution, an ILP is solved based on the columns found so far. Finally, to speed up the solution procedure, we used a restricted set of problem parameters obtained via a sampling procedure, called threshold sampling. The threshold sampling selects decision splits among the most frequent splits selected by CART. 

The computational experiments indicate that the proposed approach is competitive with the state-of-the-art ILP-based algorithms in terms of training and testing accuracy. The strength of our approach consists in its capability of finding good solutions for large instances (with more than 10000 rows) in short computational times, unlike the the ILP-based algorithms previously described in the literature (\citet{BertDunn17}, \citet{VerZha17}, and \citet{VerZha19}).

 In Proposition \ref{prop:pricing}, we give an upper bound for the polynomial complexity of the pricing problem of the proposed CG procedure. The  problem of finding an efficient algorithm of the pricing problem remains open. This algorithm may be a practical alternative for solving the MILP model of the pricing problem. 

In this paper we focused mainly on  a math-heuristic approach. A natural further research direction is the design of an efficient Branch-and-Price (BP) algorithm to solve large instances in reasonable time limits. Developing an efficient BP algorithm with strong branching rules, valid cuts, pricing algorithm with improved comlplexity, and column management will extend the interplay of Optimization and Machine Learning research.
 
As any of the ILP-based approaches, we can incorporate other learning objectives than  accuracy into our algorithm, as demonstrated in \citet{VerZha17}. For instance, our approach can be used to build a classification tree such that the false negatives in medical data is minimized. It would be interesting to study how our approach behaves with these different objectives.


\newpage

\appendix
\section*{Appendix}

\section{Hyperparameter tuning of CART}\label{app:CARTStar}
\subsection{Hyper parameter tuning procedure of CART}\label{App:CART}

CART* that is a tuned version of CART. CART* uses the best-performing hyperparameters after conducting a test on them. 
Table \ref{Table:CartHyperparams} lists the big data instances and we only performed the following pre-processing step, (1) transforming classes to integers, (2) transforming nominal string features into 0/1 features using one-hot-encoding, and (3) transforming meaningful ranked (ordinal) string features into numerical features (for instance $\{low,medium,high\}$ becomes $\{0,1,2\}$).

\begin{center}
\begin{table}[ht]
 \centering
 \footnotesize
  \caption{Tuned hyperparameters for CART*} \label{Table:CartHyperparams}
  \begin{tabular}{ll }
\hline 
\textbf{Parameter } & \textbf{Range set} \\ 
\hline
Goodness criterion & $\{\mbox{gini, entropy}\}$. \\
Minimum sample requirement & $\{0.02,0.05,0.1,0.2\}$ \\
Class weight & $\{\mbox{None, Balanced}\}$ \\
Minimum segment size at leaves & $\{\mbox{0.01, 0.05, 0.1, 0.2,1}\}$ \\
\hline
\end{tabular}
\end{table}
\end{center} 

As listed in Table \ref{Table:CartHyperparams}, the hyperparameter tuning of CART* considers 80 possible combinations of the following parameters: (i) the minimum sample requirement, ranging from 0.02, 0.05, 0.1, to 0.2, and the minimum segment size at leaves, ranging from 0.01, 0.05, 0.1, 0.2, to 1. (ii) the performance metric used to determine the best splits, including gini and entropy; and (iii) the weights given to different classes. 
The ``Balanced'' option from Scikit-learn balances classes by assigning different weights to data samples based on the sizes of their corresponding classes.  The ``None'' option does not assign any weights to data samples. 
All these options are explored by performing an exhaustive search with a 10-fold cross validation on training data.

\section{Detailed results}
The following tables refer to the average accuracy on testing. For CART* and CG, the computational time is also provided (bracketed).

\begin{table}[H]
\begin{center}
\begin{tabular}{l|c|c|c}
Dataset & $|R|$ & $|F|$ & $|T|$  \\
\hline
\multicolumn{4}{l}{\textit{Small instances} }\\
\hline
Balance-scale & 625& 4 & 3  \\
Banknote-authentification & 1372 & 4 & 2 \\
Car-evaluation & 1728 & 5 & 4 \\
Ionosphere & 351 & 34 & 2 \\
Iris & 150 & 4 & 2 \\
Monks-problems-1 & 124 & 6 & 2\\
Monks-problems-2 & 169 & 6 & 2  \\
Monks-problems-3 & 122 & 6 & 2  \\
Pima-Indians-diabetes & 768 & 8 & 2 \\
Seismic-bumps & 2584 & 18 & 2  \\
Spambase & 4601 & 57 & 2 \\
Statlog-project-landsat-satellite & 4435 & 36 & 6 \\
Tic-tac-toe & 958 & 18 & 2 \\
Wine & 178 & 13 & 3 \\
\hline
\multicolumn{4}{l}{\textit{Large instances} }\\
\hline
Default credit & 30000 & 23 & 2  \\
Hand-posture & 78095 & 33 & 5 \\
HTRU\_2 & 17898 & 8 & 2  \\
Letter recognition & 20000 & 16 & 26 \\
Magic4 & 19020 & 10 & 2  \\
Statlog shuttle & 43500 & 9 & 7 \\
\hline

\end{tabular}
\caption{Dataset instances used in the experiments}
\label{tab:instances}
\end{center}
\end{table}

\begin{table}[H]
\begin{center}
\begin{tabular}{l|ccc|ccc|ccc}

Dataset & CART & BinOCT* &  CGH & CART & BinOCT* & CGH & CART & BinOCT* & CGH   \\
\hline
 &\multicolumn{3}{l}{ $k= 2$} &  \multicolumn{3}{l}{ $k=3$} &  \multicolumn{3}{l}{ $k=4$}  \\
 \hline
Balance-scale & 71.2 & \textbf{73.3} & 73.2  & 76.5 & 78.7 &\textbf{ 78.9} & 82.9 & 84.1 & \textbf{84.4} \\
Banknote-auth. & 91.7 & \textbf{93.4} & 92.0  & 94.6  & \textbf{97.7} & 96.0  & 97.4 & \textbf{99.7} & 97.7 \\
Car-evaluation & \textbf{76.9} & \textbf{76.9} & \textbf{76.9} & 79.0 & \textbf{80.4} & 79.8  & 84.2 & \textbf{85.7} & 85.2 \\
Ionosphere & 91.0 & \textbf{91.2} & 91.1  & 93.8 & \textbf{94.9} &94.4 & 96.0 & \textbf{97.1} & 96.1 \\
Iris &\textbf{ 96.8} & \textbf{96.8} & \textbf{96.8}  & 98.1 & \textbf{100}& 98.9  & \textbf{100} & \textbf{100} & \textbf{100}\\
Monks-prob.-1 & 76.8 & \textbf{83.5} & 77.1  & 81.6 & \textbf{92.6} & 90.3  & 86.1 & \textbf{99.4} & 93.5\\
Monks-prob.-2 & 65.2 & \textbf{69.8} & 67.1  & 70.0 & \textbf{79.5} & 78.5  & 79.8 & \textbf{86.9} & 81.4 \\
Monks-prob.-3 & \textbf{93.8} & \textbf{93.8} & \textbf{93.8} & 94.8 & \textbf{95.7} & 95.1  & 95.7 & \textbf{98.0} & 95.7\\
PI-diabetes & 77.3 & \textbf{79.3} & 78.8  & 78.9 & \textbf{81.3} & 81.2  & 82.9 & \textbf{84.7} & 84.2\\
Seismic-bumps & 93.1 & \textbf{93.4} & 93.3  & 93.4 & \textbf{93.7} & \textbf{93.7}  & 93.9 & \textbf{94.2} & \textbf{94.2} \\
Spambase & 86 & 86.7 & \textbf{87.1}  & 89.6 & 90.2 & \textbf{90.3} & 91.6 & \textbf{91.9} & 91.6\\
Statlog-sat. & 63.2 & \textbf{66.7} & 64.0  & 78.7 & \textbf{80.5} & 79.5 & 81.6 & 81.6 & \textbf{82.9}\\
Tic-tac-toe & 71.2 & \textbf{72.1} & 71.8 & 75.4 & \textbf{77.6} & 76.7  & 84.4 & 85.3 & \textbf{85.4}\\
Wine & 95.7 &\textbf{ 97.3} & 96.6  & 99.3 & \textbf{100} & 99.6  & \textbf{100} &\textbf{100} & \textbf{100}\\
\hline
\end{tabular}
\caption{Results on training, small datasets, for depths 2,3, and 4 }
\label{tab:smalldatatr}
\end{center}
\end{table}

\begin{table}[H]
\begin{center}
\begin{tabular}{l|ccccc}
Dataset & CART & CART(BD)* & OCT & BinOCT* &  CGH  \\
\hline
Balance-scale & 67.5 & 64.5 & 67.1 & \textbf{69.3} & 69.2 \\
Banknote-auth. & 90.6 & 89.0 & 90.1 & \textbf{91.7} & 91.4 \\
Car-evaluation & \textbf{77.8} & 73.7 & 73.7 & \textbf{77.8} & \textbf{77.8} \\
Ionosphere & 87.7 & \textbf{87.8} & \textbf{87.8} & 87.7 & 86.4 \\
Iris & \textbf{95.8} & 92.4 & 92.4 & \textbf{95.8} & \textbf{95.8} \\
Monks-prob.-1 &  68.4 & 57.4 & 67.7 & \textbf{80.0} & 71.0\\
Monks-prob.-2 & 54.0 &\textbf{ 60.9} & 60.0 & 54.4 & 52.6\\
Monks-prob.-3 & 93.5 &\textbf{ 94.2} & \textbf{94.2} & 93.5 & 93.5 \\
PI-diabetes & 74.7 & 71.9 & 72.9 & 73.1 &\textbf{ 75.2 }\\
Seismic-bumps & \textbf{94.0} & 93.3 & 93.3 & 93.8 & 93.6 \\
Spambase & 85.4 & 84.2 & 84.3 & 85.7 & \textbf{86.5} \\
Statlog-sat. & 63.4 & 63.2 & 63.2 & \textbf{65.7} & 63.9\\
Tic-tac-toe & 67.2 & 68.5 &  \textbf{69.6} & 67.3 & 67.7 \\
Wine & 88.0 & 81.3 & \textbf{91.6} & 91.1 & 87.6 \\
\hline
\multicolumn{6}{l}{* results reported by \citet{BertDunn17}}. 
\end{tabular}
\caption{Results on testing, small datasets, for depth 2 }
\label{tab:test2}
\end{center}
\end{table}

\begin{table}[H]
\begin{center}
\begin{tabular}{l|ccccc}
Dataset & CART & CART(BD)* & OCT & BinOCT* &  CGH  \\
\hline
Balance-scale & 70.6 & 70.4 & 68.9 & 71.3 & \textbf{72.4} \\
Banknote-auth. & 93.6 & 89.0 & 89.6 & \textbf{96.6} & 95.2 \\
Car-evaluation & 78.9 & 77.4 & 77.4 & \textbf{80.4} & 79.0 \\
Ionosphere & 86.4 & \textbf{87.8} & 87.6 & 85.5 & 86.4 \\
Iris & 95.8 & 92.4 & 93.5 & \textbf{97.9} & 95.4  \\
Monks-prob.-1 &  76.8 & 65.8 & 74.2 & 80.0 & \textbf{92.2} \\
Monks-prob.-2 & 56.7 & \textbf{60.9} & 60.0 & 55.3 & 56.3\\
Monks-prob.-3 & 92.3 & \textbf{94.2} & \textbf{94.2} & 89.7 & 90.3 \\
PI-diabetes & 73.3 & 70.6 & 71.1 & 74.4 & \textbf{75.6}  \\
Seismic-bumps & 93.1 & \textbf{93.3} & \textbf{93.3} & 92.8 & 92.9 \\
Spambase & 88.5 & 86.0 & 86.0 & \textbf{88.9} & 88.8 \\
Statlog-sat. & 77.3 & 77.7 & 77.9 & \textbf{79.2} & 77.7\\
Tic-tac-toe & 73.8 & 73.1 & \textbf{74.1} & 70.6 & 71.3 \\
Wine & 88.0 & 80.9 & \textbf{94.2} & 92.0 & 88.0  \\
\hline
\multicolumn{6}{l}{* results reported by \citet{BertDunn17}}. 
\end{tabular}
\caption{Results on testing, small datasets, for depth 3 }
\label{tab:test3}
\end{center}
\end{table}

\begin{table}[H]
\begin{center}
\begin{tabular}{l|ccccc}
Dataset & CART & CART(BD)* & OCT & BinOCT* &  CGH  \\
\hline
Balance-scale & 77.5 & 73.4 & 71.6 & 78.9 & \textbf{79.1} \\
Banknote-auth. & 95.8 & 89.0 & 90.7 & \textbf{98.1} & 95.7 \\
Car-evaluation & 84.8 & 78.8 & 78.8 & \textbf{86.5} & 86.0 \\
Ionosphere & 87.5 & 87.8 & 87.6 & \textbf{88.6} & 85.7 \\
Iris & 97.9 & 92.4 & 93.5 & \textbf{98.4} & 97.9  \\
Monks-prob.-1 &  74.2 & 68.4 & 74.2 & \textbf{87.1} & 84.5 \\
Monks-prob.-2 & \textbf{63.6} & 62.8 & 54.0 & 63.3 & 62.3\\
Monks-prob.-3 & 93.5 & \textbf{94.2} & \textbf{94.2} & 84.5 & 89.0 \\
PI-diabetes & 73.9 & 71.1 & 72.4 & 73.0 & \textbf{75.1}  \\
Seismic-bumps & 92.6 &\textbf{ 93.3} & \textbf{93.3} & 92.6 & 92.5  \\
Spambase & \textbf{89.7} & 86.0 & 86.1 & 89.5 & 89.6 \\
Statlog-sat. & 79.9 & 78.2 & 78.0 & 79.9 & \textbf{80.0}\\
Tic-tac-toe &\textbf{80.1} & 74.2 & 73.3 & 78.8 & 79.3 \\
Wine & 88.9 & 80.9 & \textbf{94.2} & 89.8 & 89.8  \\
\hline
\multicolumn{6}{l}{* results reported by \citet{BertDunn17}}. 
\end{tabular}
\caption{Results on testing, small datasets, for depth 4 }
\label{tab:test4}
\end{center}
\end{table}

\section{Complexity of DPP}\label{app:Complexity}

In the following we study the complexity of a special case of the pricing problem, in which all dual values $\alpha$'s and $\gamma$'s  are set to 0, and all the $\beta$ variables are set to  $-1$. 
We call this special case the ``Decision pricing problem (DPP)''.  

\begin{quote}
	{\sc Problem:} {\sc Decision Path Pricing problem (DPP) }
	\\[1.0ex]
	{\sc Instance:} A binary tree $BT$, a set of data rows $R$,  a set of features $F$, a leaf node $l$,  a set of splits  $S(j)$  for every $j$ in $p_{BT}(l)$. A real number $b$.
	\\[1.0ex]
	
	{\sc Question:}
	Does there exist a decision path $p$ in $DP_{l}$ such that  $\overline{\mbox{CP}}(p) \geq b$, where $\overline{\mbox{CP}}(p)$ is given by (\ref{redcost})?
\end{quote}

\begin{theorem} The (DPP) problem for arbitrary depth $k$ is strongly NP-hard. 
\end{theorem}

\begin{proof}
	The proof uses a reduction from Exact Cover by 3-Sets(3XC) to DPP. 3XC is a well-known NP-complete problem in the strong sense \cite{NPcomp}.
	\newline
	
	\textit{Exact Cover by 3-Sets:} Given a set $X=\{1,\dots,3q\}$ and a collection $\mathcal{C}$ of 3-element subsets of $X$, does there exist a subset $\mathcal{C'}$ of $\mathcal{C}$ where every element of $X$ occurs in exactly one member of $\mathcal{C'}$?\\
	
	Given an instance $I$ of the $3XC$ problem, we now present a polynomial time transformation to an instance $I'$ of the DPP problem. By the definition of a decision path, all the splits at internal nodes have to be distinct.
	
	\begin{itemize}
		\item \textit{Rows and compatibility:} For every element in $\mathcal{C}$ we create a distinct row, so $|R|=|\mathcal{C}|$. We say that two rows $r$ and $r'$ are \textit{compatible} if the corresponding elements in $\mathcal{C}$ are disjoint, and it is denoted by $r \propto r'$.
		\item \textit{Features and feature values:} For every row $r$ in $R$, we define a distinct feature $f_{r}$.  This implies $|F| = |R|$. For each row $r$, the value of a feature $f_{r'}$ is defined as
		
		\[ v_{r}^{f_{r'}}= \begin{cases} 
		0.5 & r \propto r' \mbox{ or } r=r'  \\
		1 & r \not \propto r'
		\end{cases}, \quad r,r' \in R
		\]
		
		\item \textit{Leaf, depth, splits:} Consider a binary tree $BT$ of depth $q$. Let $p_{BT}(l)=(n(0),..., n(q-1),l)$ be the path in $BT$  where $n(h)$ is the left child of $n(h-1)$, for $h\in \{1,..., q-1\}$ and $n(0)$ is the root node. Note that  $|Sq_{l}| = q$ (recall $|X|=3q$). At every node $ j \in p_{BT}(l)\cap N_{int}$, define the set of splits $S(j)=\{(f_r, v^{f_{r}}\leq 0.5): r \in R \}$.    
		
		\item Choose $b=q$. In the considered special case, $\overline{\mbox{CP}}(p)$ actually counts the number of rows reaching $l$ when the decision path is $p$. Consequently, $\max\limits_{p \in DP_l}\overline{\mbox{CP}}(p)$ cannot be greater than $q$, which is equal to the highest number of compatible rows. Hence, the question can be reformulated as ``\textit{Does there exist a decision path $p$ that directs exactly $q$ rows to the leaf node $l$?}''.
	\end{itemize}
	
	Let $I$ be a YES instance for 3XC. Now let $R^{\mathcal{C'}}$ denote the set of rows corresponding to the elements in $\mathcal{C'}$. Note that since the elements of  $\mathcal{C'}$ are disjoint, these rows are compatible. Next select at each  internal node $j\in  p_{BT}(l)\cap N_{int}$  exactly one  split  $(f_{r},v^{f_{r}}\leq 0.5)$  for every $r$ in $R^{\mathcal{C'}}$. Observe that each row  $r$ in $R^{\mathcal{C'}}$ appears in  exactly $q$ splits  $\{(f_{r'},v^{f_{r'}}\leq 0.5) : r'\in R^{\mathcal{C'}}\}$, implying that  either $r \propto r'$ or $r=r'$. Moreover, $r$ is directed left at all internal nodes due to feature values $v_{r}^{f_{r'}}=0.5$ and therefore reaches leaf $l$. The decision path constructed in this way  is a YES instance to the decision version of DPP.  The other direction is trivial, since the subsets of $\mathcal{C}$ corresponding to the $q$ rows that reach leaf $l$  give an exact cover  for the $3XC$ instance $I$.

\end{proof}

\end{document}